\documentclass[twoside]{article}

\usepackage[accepted]{aistats2025}
\usepackage{amsmath}
\usepackage{amssymb}
\usepackage{mathtools}
\usepackage{amsthm}
\usepackage{pifont}

\usepackage{bm}
\usepackage{mathtools}
\usepackage{xcolor}
\usepackage{upgreek}
\usepackage{tikz}
\usetikzlibrary{arrows.meta, angles, quotes, babel, calc, shapes.misc}
\usepackage{pgfplots}
\usepackage{calrsfs}
\usepackage{amssymb}
\usepackage{csvsimple}
\usepackage{pgfplotstable}
\usepackage{booktabs}
\usepackage{makecell}
\usepackage{cellspace}
\usepackage{multicol}
\usepackage{comment}
\usepackage{threeparttable}
\usepackage{nccmath}
\usepackage[export]{adjustbox}
\usepackage{transparent}
\usepackage{scalerel}
\usepackage{enumitem}
\usepackage{amsthm}
\usepackage{tabularx}
\usepackage{algpseudocode,algorithm}
\usepackage{float}
\usepackage{hyperref}

\DeclareMathAlphabet{\pazocal}{OMS}{zplm}{m}{n}

\DeclareMathOperator*{\argmin}{argmin}

\newcommand{\norm}[1]{\left\lVert#1\right\rVert}

\theoremstyle{plain}
\newtheorem{theorem}{Theorem}[section]
\newtheorem{proposition}[theorem]{Proposition}
\newtheorem{lemma}[theorem]{Lemma}

\theoremstyle{definition}
\newtheorem{definition}[theorem]{Definition}

\theoremstyle{remark}

\usepackage[capitalize,noabbrev]{cleveref}
\usepackage{subfigure}

\usepackage[round]{natbib}
\makeatletter
\newcommand{\bigcomp}{%
	\DOTSB
	\mathop{\vphantom{\sum}\mathpalette\bigcomp@\relax}%
	\slimits@
}
\newcommand{\bigcomp@}[2]{%
	\begingroup\m@th
	\sbox\z@{$#1\sum$}%
	\setlength{\unitlength}{0.75\dimexpr\ht\z@+\dp\z@}%
	\vcenter{\hbox{%
			\begin{picture}(1,1)
				\bigcomp@linethickness{#1}
				\put(0.5,0.5){\circle{1}}
			\end{picture}%
	}}%
	\endgroup}
\newcommand{\bigcomp@linethickness}[1]{%
	\linethickness{%
		\ifx#1\displaystyle 2\fontdimen8\textfont\else
		\ifx#1\textstyle 1.65\fontdimen8\textfont\else
		\ifx#1\scriptstyle 1.65\fontdimen8\scriptfont\else
		1.65\fontdimen8\scriptscriptfont\fi\fi\fi 3
	}%
}
\makeatother
 
\begin{document}

%

%

\twocolumn[

\aistatstitle{Learning Geometrically-Informed Lyapunov Functions with Deep Diffeomorphic RBF Networks}

\aistatsauthor{ Samuel Tesfazgi \And Leonhard Sprandl \And  Sandra Hirche }

\aistatsaddress{ TU Munich \And TU Munich \And TU Munich } ]

\begin{abstract}
The practical deployment of learning-based autonomous systems would greatly benefit from tools that flexibly obtain safety guarantees in the form of certificate functions from data. While the geometrical properties of such certificate functions are well understood, synthesizing them using machine learning techniques still remains a challenge. To mitigate this issue, we propose a diffeomorphic function learning framework where prior structural knowledge of the desired output is encoded in the geometry of a simple surrogate function, which is subsequently augmented through an expressive, topology-preserving state-space transformation. Thereby, we achieve an indirect function approximation framework that is guaranteed to remain in the desired hypothesis space. To this end, we introduce a novel approach to construct diffeomorphic maps based on RBF networks, which facilitate precise, local transformations around data. Finally, we demonstrate our approach by learning diffeomorphic Lyapunov functions from real-world data and apply our method to different attractor systems. 
\end{abstract}

\section{INTRODUCTION}
\label{sec:introduction}
With recent advances in robotics and machine learning, data-driven autonomous systems are increasingly deployed in safety-critical application scenarios such as autonomous driving \citep{liu2024augmenting} or robotic rehabilitation \citep{Qingsong2023}. While learning-based systems are particularly well-suited for such complex and uncertain environments, a limitation that inhibits their deployment is the lack of formal safety and stability guarantees. A practical method to ascertain the desired safety and stability properties of a dynamical system is through the construction of certificate functions, e.g., a Lyapunov function to show convergence to an equilibrium point \citep{book:khalil:nonlin}. 
A major strength of these approaches is the existence of converse theorems, i.e., if the desired property holds, a certificate function is guaranteed to exist \citep{Teel14, Liu22}. \textcolor{black}{Besides certification, Lyapunov functions may also be utilized for control synthesis \citep{SONTAG1989117, Tesfazgi2024} and have additionally been deployed as value function approximators in the context of reinforcement learning \citep{Chow2018, Chang19, Tesfazgi2021}.} 

In general, certificate functions express the long-term behavior of a system's trajectory through invariant set constraints. Thereby, the set of states to which the system is bounded or converges to, is geometrically encoded in the level sets of the certificate function. A candidate Lyapunov function, for instance, has to be positive-definite with a strictly decreasing time-derivative. While these conditions can be resolved in simple settings, e.g., when the dynamics are known, and the hypothesis space is limited to sum-of-squares polynomials \citep{parrilo00}, no constructive approach is known for general, nonlinear systems. Therefore, the need for expressive learning techniques that construct certificate functions directly from data arises. \looseness=-1 

Recently, the deployment of neural networks (NNs) has been proposed to learn Lyapunov functions from observations \citep{Richards2018, Ravanbakhsh19, Chang19, Manek19, dawson2021}. 
However, even though NNs have the advantage of strong representational capabilities, imposing the necessary constraints is an open issue.  
Existing methods either induce the Lyapunov conditions via soft-constraints \citep{Chang19}, only admitting empirical statements, or strictly by extensively searching for counter-examples \citep{Ravanbakhsh19}, which is computationally demanding. A promising perspective is to geometrically constrain the output 
by using a suitable architecture \citep{RAISSI2019686}. However, 
the imposed output constraints are either not specific enough, only guaranteeing positive definiteness \citep{Richards2018, dawson2021}, or overly conservative, e.g., using input convex NNs \citep{amos2017icnn, Manek19}. 

\textbf{Contribution.} In this work, we follow an alternative approach of encoding structural knowledge and imposing desired geometric properties on the inferred function by deploying smooth and bijective maps, so-called diffeomorphisms \citep{book:boumal23}. In particular, instead of constraining the output of a function approximator directly, we specify a simple base function with desired geometric properties and subsequently learn a topology-preserving, state-space transformation under which the augmented base function adheres to the data, thereby indirectly obtaining a Lyapunov function. \textcolor{black}{To facilitate the synthesis, we propose a novel diffeomorphism model based on RBF networks in this paper. Our kernel-based architecture enables the generation of precise, local transformations, which help induce constraint satisfaction at the data points.} Beyond point attractors, we additionally demonstrate the applicability of our approach for more general system classes, including multiple equilibria and limit cycles.

While the regularity-preserving properties of diffeomorphisms have been used in the context of imitation learning \citep{Rana20} and control \citep{Sun23}, their utilization for learning certificate functions remains understudied. \textcolor{black}{In particular, since the existence of a diffeomorphic map between a linear base system and a nonlinear, data-generating system is generally not guaranteed \citep{Bevanda22}, the problem of learning diffeomorphic certificate functions directly from data remains highly relevant and warrants separate consideration.}

\section{PRELIMINARIES}
\textbf{Lyapunov stability theory.} Consider an autonomous system\footnote{\textit{Notation:} Lower and upper case bold symbols denote vectors and matrices, 
$\mathbb{R}_+$/$\,\mathbb{N}_+$ all real/integer positive numbers, $\bm{I}_n$ the $n\times n$ identity matrix, $\|\cdot\|$ the Euclidean norm, $|\cdot|$ the absolute value, and $\#(\cdot)$ the cardinality of a set. $\pazocal{L}$ is the set of Lipschitz continuous functions, $\pazocal{D}$ the set of diffeomorphic maps, and $p$-times continuously differentiable functions are denoted by $\pazocal{C}^p$. The composition of two functions is written $f\circ g=f(g(\cdot))$ and the nested composition of functions is denoted by $\bigcomp_{n=1}^{N}f_n = f_N \circ \cdots \circ f_1$. } \looseness=-1
\begin{equation}
	\label{eq:system}
	\dot{\bm{x}} = f(\bm{x}),
\end{equation}
with continuous state $\bm{x} \in \mathbb{R}^n$ and system dynamics ${\bm{f}\colon \mathbb{R}^n \to \mathbb{R}^n}$. The problem of certifying stability is concerned with analyzing the behavior of $\bm{x}(\tau)$ for time $\tau \to \infty$, given some initial state $\bm{x}(\tau_0)=\bm{x}_0$.  In order to formalize this property, we introduce the following concept of stability. \looseness=-1
\begin{definition}[\cite{book:khalil:nonlin}]
\label{def:stab}
    A system \eqref{eq:system} has an asymptotically stable equilibrium $\bm{x}^*$ on the set $\pazocal{X}$ if
    \begin{enumerate}[topsep=0pt,itemsep=2pt]
        \item 
        for all $d \!>\! 0$, there exist  $\delta \!>\! 0$, $\tau_0\!\geq\! 0$ such that $\norm{\bm{x}_0 \!-\! \bm{x}^*} \!<\! \delta$ implies $\norm{\bm{x}(\tau) \!-\! \bm{x}^*} \!<\! d $, $\forall \tau \!\geq\! \tau_0.$
        \item 
        $\lim_{\tau \to \infty}\norm{\bm{x}(\tau) \!-\! \bm{x}^*} \!=\! 0 $ for all $\bm{x}_0\in\pazocal{X}$.
    \end{enumerate}
\end{definition}
If the conditions hold for all states, i.e., $\bm{x}_0\in\mathbb{R}^n$, the equilibrium $\bm{x}^*$ is globally asymptotically stable. Without loss of generality, we assume $\bm{x}^* = \bm{0}$ from now on.  A practical method to ascertain the convergence property of a system, without solving the underlying dynamics equations, is by means of Lyapunov stability theory.
\begin{theorem}[Lyapunov Stability Theorem, \cite{book:khalil:nonlin}]
	\label{thm:lyapunov}
	Let $\bm{x}^* = \bm{0}$ be an equilibrium point for~\eqref{eq:system} and $\pazocal{X} \subset \mathbb{R}^n$ be the domain of $f: \pazocal{X} \mapsto \mathbb{R}^n$ with $\bm{x}^* \in \pazocal{X}$. Let~${V: \pazocal{X} \mapsto \mathbb{R}}$ be a continuously differentiable function such that:
	\begin{subequations}
		\begin{align}
			\label{thm:lyapunov:L0}
			V(\bm{0}) &= 0 \\	
			\label{thm:lyapunov:L1}
			V(\bm{x}) &> 0 \quad \forall \bm{x} \in \pazocal{X}\setminus\{\bm{0}\}\\
			\label{thm:lyapunov:L2}
			\dot{V}(\bm{x}) = \nabla_{x}^\intercal V(\bm{x}) f(\bm{x}) &< 0 \quad \forall \bm{x} \in \pazocal{X}\setminus\{\bm{0}\}
		\end{align}
	\end{subequations}
	Then, $\bm{x}^*$ is locally asymptotically stable in the sense of Definition \ref{def:stab}. 
\end{theorem}
Thus, finding a function $V(\cdot)$ that satisfies \eqref{thm:lyapunov:L0}-\eqref{thm:lyapunov:L2} is sufficient to certify stability of $f(\cdot)$.

\textbf{Multiple Equilibria and Limit Cycles.} 
Beyond asymptotic stability to a single equilibrium, a dynamical system may also exhibit other types of attractor landscapes, such as multiple equilibria, where system trajectories converge to different states out of a set 
\begin{equation}
    {\pazocal{X}^* := \{\bm{x} \in \mathbb{R}^n \: | \: f(\bm{x}) = \bm{0}\}} \nonumber
\end{equation}
depending on the initial state $\bm{x}_0$. Another common attractor are limit cycles, which describe invariant sets $\pazocal{X}^\circ$ under the dynamics $f(\cdot)$ for some orbital period $T$ 
\begin{equation}
	\label{def:setlimitcycle}
	\pazocal{X}^\circ := \{\bm{x} \:|\: \bm{x}(\tau) \!=\! \bm{x}(\tau + T), \, f(\bm{x}) \!\neq\! \bm{0}, \; \forall \tau \geq 0, \, \exists T > 0 \}. \nonumber
\end{equation}
In order to extend the notion of Lyapunov stability analysis to such systems, it is common to introduce a \textit{Lyapunov-like} function \citep{patrao2011existence, bjornsson2015computation} that satisfies the conditions \eqref{thm:lyapunov:L0}-\eqref{thm:lyapunov:L2} for the respective sets $\pazocal{X}^*$ or $\pazocal{X}^\circ$, instead of only $\{\bm{0}\}$. For notational convenience, we collectively use 
$\pazocal{X}^{\bm{0}}$ for all convergence sets in the following.

\textbf{Diffeomorphism.} 
A mapping ${\phi: \mathbb{R}^n \mapsto \mathbb{R}^n}$ is bijective, if it's inverse $\phi^{-1}(\cdot)$ is guaranteed to exist. If the mapping $\phi(\cdot)$ and its inverse $\phi^{-1}(\cdot)$ are further smooth, it is referred to as a \textit{diffeomorphism}, defined as follows. 
\begin{definition}[\cite{book:boumal23}]
	\label{def:diffeo}
	A diffeomorphism $\phi : U \to V$ with open sets $U, V \subseteq \mathbb{R}^n$ is a bijective and smooth map whose inverse $\phi^{-1}$ is also smooth. 
\end{definition}
We denote the set of diffeomorphic maps with $\phi\in\pazocal{D}$. The requirement of $\phi(\cdot)$ being smooth allows the mapping between two \textit{differentiable manifolds}. Since a differentiable manifold is additionally equipped with a differential structure~\citep{book:lee:smooth}, it gives rise to the tangent space required to define gradients, which are necessary for any gradient-based analysis framework, such as Lyapunov stability analysis. Conveniently, diffeomorphic maps preserve the topology of objects, such as functions or differential equations. Intuitively, two sets 
are \textit{topologically equivalent}, if a mapping between the two can be established, with the map and its inverse being continuous~\citep{book:lee:manifolds}. \Cref{fig:diffeogrid} illustratively depicts the difference between a topology-preserving and a non-topology-preserving transformation. \looseness=-1

\begin{figure}
	\centering
	\subfigure[State space $\bm{x}$]{
		\includegraphics[trim={3.3cm 2.5cm 3.3cm 3.3cm}, clip, width=0.132\textwidth] {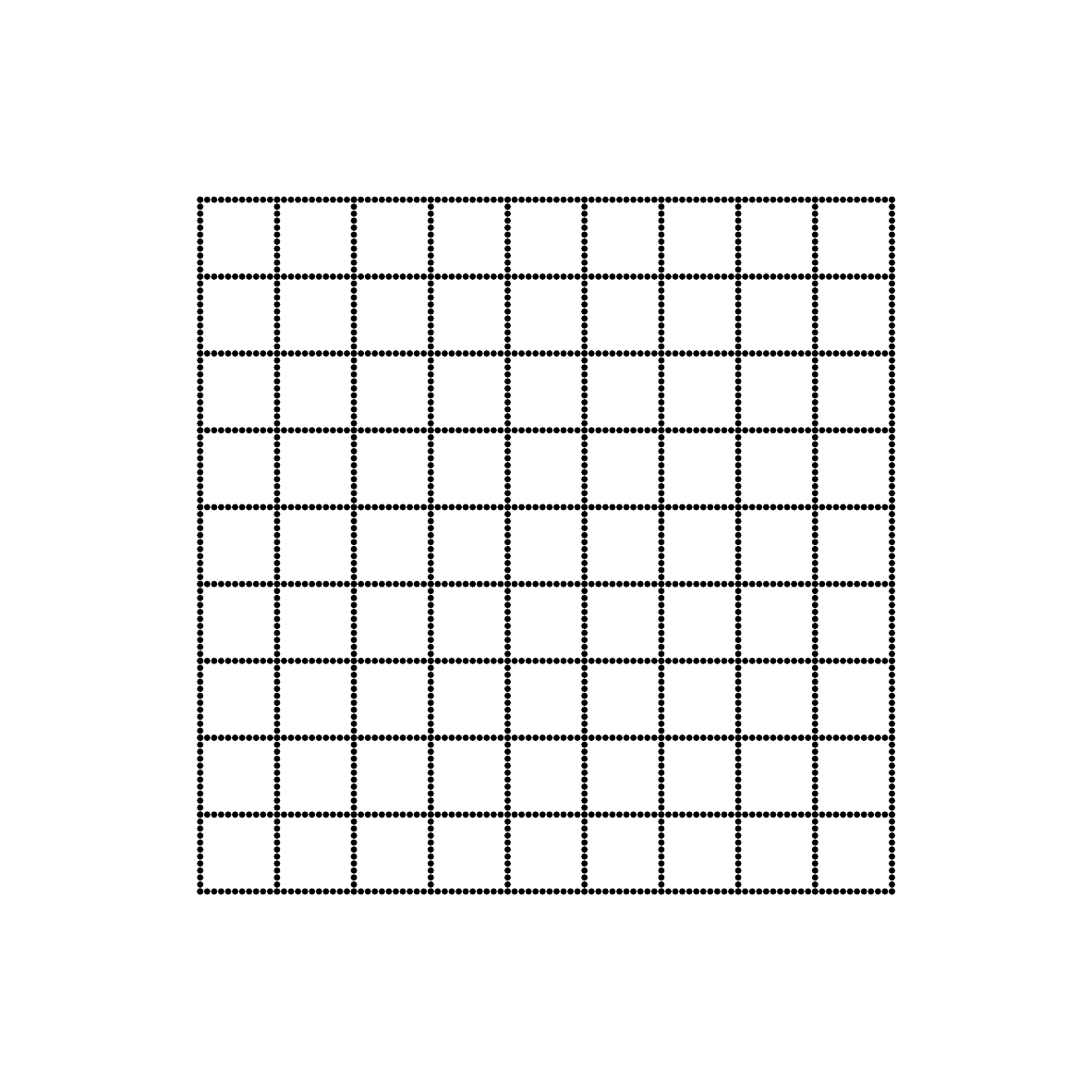}
		\label{fig:diffeogrid:original}
	}
	\hfill
	\subfigure[Invertible]{
		\includegraphics[trim={3cm 2cm 3cm 1cm}, clip, width=0.132\textwidth] {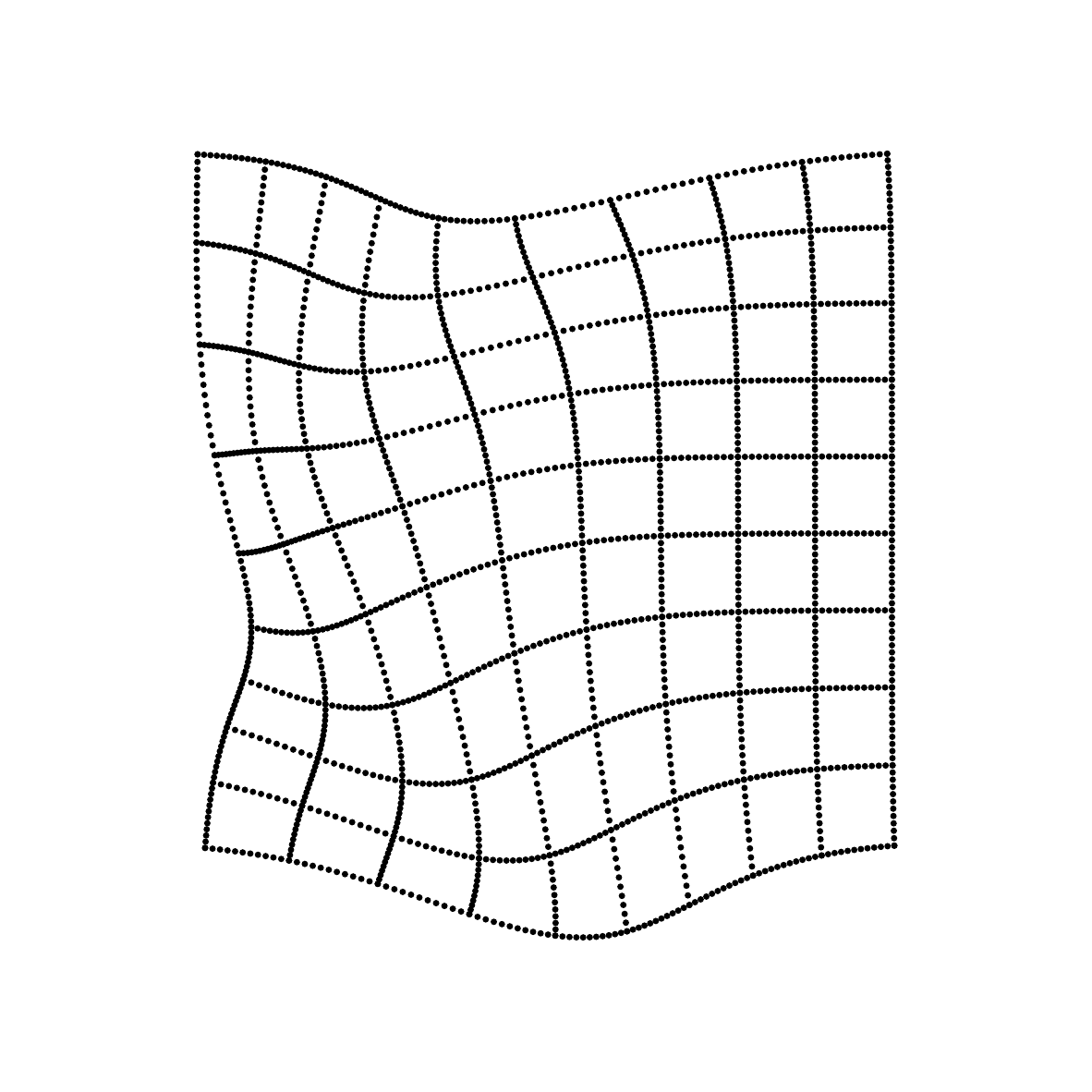}
		\label{fig:diffeogrid:morphed}
	}
	\hfill
	\subfigure[Non-invertible]{
		\includegraphics[trim={3.5cm 2cm 2cm 1cm}, clip, width=0.132\textwidth] {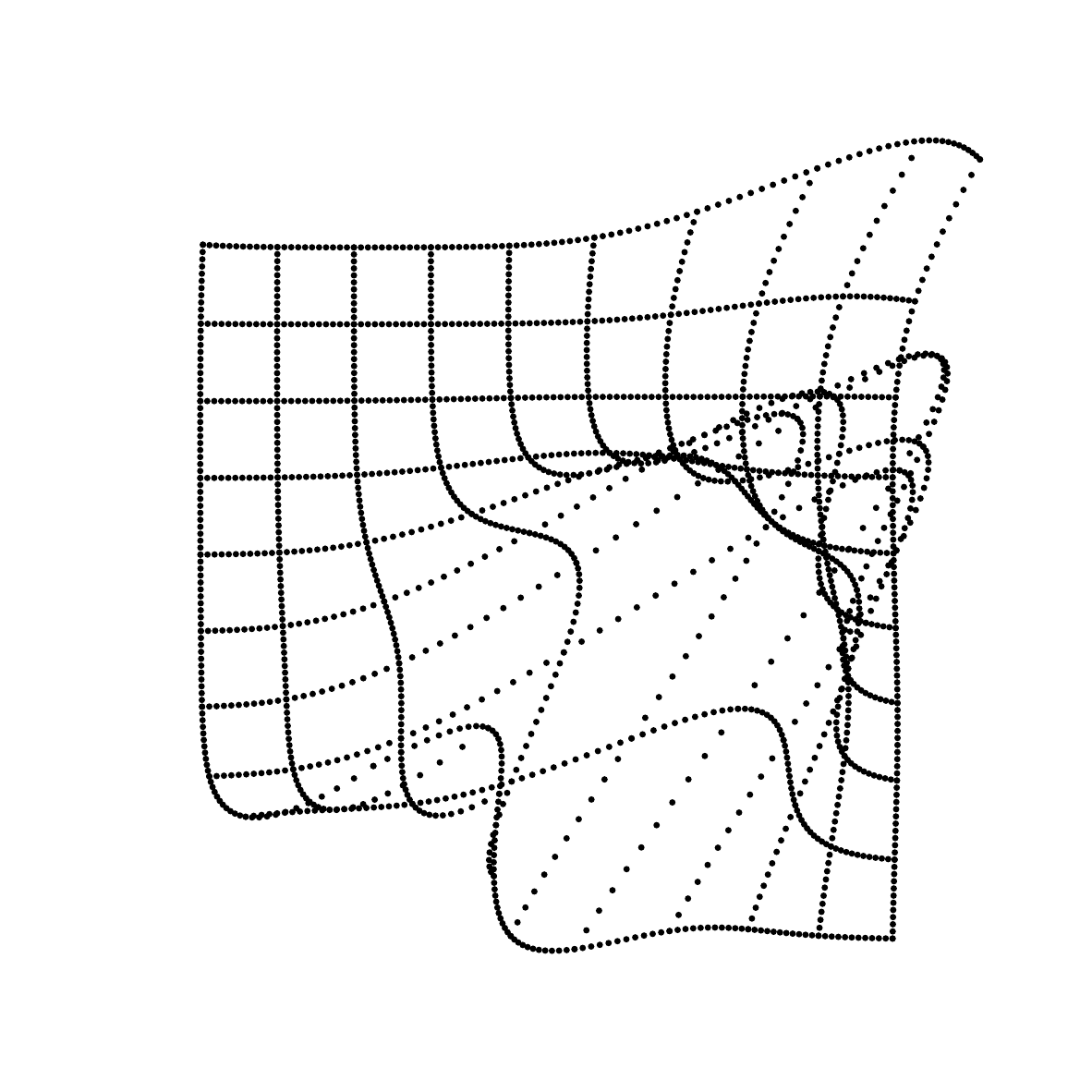}
		\label{fig:diffeogrid:violation}
	}
	\caption[Invertible and Non-Invertible transformations]{Transformations of a 2D space (a) by an invertible (b) and a non-invertible map (c). Invertability requires that the map is unique for each point.}
	\label{fig:diffeogrid}
\end{figure}

\section{DIFFEOMORPHIC LYAPUNOV FUNCTIONS}
Directly searching for a Lyapunov function is difficult since constraints \eqref{thm:lyapunov:L1} and \eqref{thm:lyapunov:L2} need to hold for an uncountable, infinite set of states. To overcome this, we propose to exploit the topological equivalence of Lyapunov functions shown by \cite{paper:gruene99} to reformulate the function approximation problem to an optimization over state-space transformations. 

\subsection{Lyapunov Function Hypothesis Space}
The primary challenge in synthesizing a Lyapunov function is guaranteeing that the gradient $\nabla_{x}V(\cdot)$ fulfills the descent condition \eqref{thm:lyapunov:L2}. Typically, the dynamics $f(\cdot)$ are not known analytically. Thus, in our considered scenario we only have access to trajectory samples $\mathbb{D}=\{\bm{x}^{(i)}, \dot{\bm{x}}^{(i)}\}_{i=1}^N$. However, we may still derive shape-constraints that any potential Lyapunov function candidate has to adhere to locally, since
\begin{equation}
\label{eq:nonvanishing_g}
    \nabla_{x}^\intercal V(\bm{x}) f(\bm{x})\! <\! 0\!\! \implies \!\!\! \nabla_{x}V(\bm{x}) \!\neq \!\bm{0}, \enspace \forall \bm{x} \!\in\! \pazocal{X}\setminus\!\{\pazocal{X}^{\bm{0}}\}.
\end{equation}
Consequently, a positive definite function $V(\cdot)$ with non-vanishing gradient $\nabla_{x}V(\bm{x}) \neq \bm{0}$ is always a valid Lyapunov function for \textit{some} stable system. Therefore, a diffeomorphic transformation that preserves these topological properties is guaranteed to generate an output that remains in the space of Lyapunov functions, which we demonstrate in the following: \looseness=-1
\begin{proposition}
	\label{prop:topoconservation}
	Consider a smooth function $V : \mathbb{R}^n \mapsto \mathbb{R}$ and let $N_V$ denote the number of unique gradient roots of $V(\cdot)$ \looseness=-1
	\begin{equation}
		N_V = \#(\pazocal{S}_V), \quad \text{with }\; \pazocal{S}_V = \{\bm{x} | \nabla_{\bm{x}}V(\bm{x}) = \bm{0}\}
	\end{equation}
	where $\#(\pazocal{S})$ denotes the cardinality of the set $\pazocal{S}$. Next, let $\phi: \mathbb{R}^n \mapsto \mathbb{R}^n$ be an orientation-preserving diffeomorphism, i.e., its Jacobian $\bm{J}_{\phi}\in \mathbb{R}^{n \times n}$ satisfies
	\begin{equation}
		\label{eq:prop:topoconv:pdfdet}
		\det(\bm{J}_\phi(\bm{x})) > 0 \quad \forall \bm{x} \in \mathbb{R}^n.
	\end{equation}
	Then, the number of gradient roots is invariant under the map $\phi(\cdot)$ and $N_V = N_U$, where $U := V \circ \phi$.
\end{proposition}
\begin{proof}
	Due to~\eqref{eq:prop:topoconv:pdfdet}, $\bm{J}_\phi$ is full rank~$\forall \bm{x} \in \mathbb{R}^n$, and consequently, the nullspace $\text{null}(\bm{J}_\phi(\bm{x}))$ only contains the trivial solution~\citep{strang2019linear}. 
    The same holds for the transpose, since ${\det(\bm{J}_\phi(\bm{x})) = \det(\bm{J}_\phi(\bm{x})^\top) > 0}$. Thus, the \textit{left} nullspace~\citep{strang2019linear} also only contains the trivial solution. 
	Applying the chain rule 
	\begin{equation}
		\label{eq:morphgrad}
		\nabla_x U(\bm{x}) = \frac{\partial}{\partial\bm{x}} V(\phi(\bm{x})) = \bm{J}_{\phi}(\bm{x})^\top\nabla_{\bm{x}}V(\bm{x}).
	\end{equation}
	From \eqref{eq:prop:topoconv:pdfdet} and \eqref{eq:morphgrad}, it trivially follows that ${\#(\pazocal{S}_V) = \#(\pazocal{S}_U)}$, which concludes the proof.
\end{proof}
Thus, it follows that candidate Lyapunov functions are diffeomorphic to one another and consequently that any Lyapunov functions can be transformed into a simple $\pazocal{K}_\infty$ function under a change of coordinates, as proposed in \cite{paper:gruene99}. For a single point attractor system for instance, 
the time derivative of a valid Lyapunov function $V(\cdot)$ has to decrease along the trajectories, thereby necessitating non-vanishing gradients outside of the equilibrium \eqref{eq:nonvanishing_g}. Therefore, each contour line of any $V(\cdot)$ is topologically equivalent to a sphere, hence, admitting a diffeomorphic transformation to one another. 

\subsection{Formulation as Diffeomorphic Learning Problem}
We propose to search over the space of Lyapunov functions by finding an appropriate diffeomorphic transformation without the need to explicitly incorporate shape constraints. The data-driven, diffeomorphic Lyapunov learning problem is formalized as follows:
\begin{definition}
	Given a dataset $\mathbb{D}=\{\bm{x}^{(i)}, \dot{\bm{x}}^{(i)}\}_{i=1}^N$ generated by an unknown stable system~${\dot{\bm{x}} = f(\bm{x})}$ 
    and any initial Lyapunov-like function $V_b\colon\mathbb{R}^n\to\mathbb{R}$ with
    \begin{align}
        \label{eq:top_lyap}
        V_b(\bm{x}) > 0 \;\land\; \nabla_{x}V_b(\bm{x}) \neq \bm{0}, \enspace \forall \bm{x}\in\! \pazocal{X}\setminus\!\{\pazocal{X}^{\bm{0}}\},
    \end{align}
    find a diffeomorphism 
	\label{def:main:dclop}
	\begin{subequations}
        \label{eq:dclop}
		\begin{align}
			\label{eq:dclop:obj}
			\phi^* &= \argmin_{\phi \in \pazocal{D}} L(V_\phi, \mathbb{D}) \\
			\label{eq:dclop:L2}
			&\text{ s.t.} \, \nabla_{\bm{x}}^\intercal V_{\phi}(\bm{x}_t) \dot{\bm{x}}_t < 0 \quad \forall t \in [1, \dots, T] 
		\end{align}
	\end{subequations}
	where $V_\phi(\bm{x}) \!=\! V_b\circ\phi(\bm{x})$ with loss function $L(\cdot)$. 
\end{definition}
Intuitively, \eqref{eq:dclop} reformulates the search for a Lyapunov function, which is a functional optimization problem, as a diffeomorphic optimization problem. First, a simple base function $V_b(\cdot)$ is specified that adheres to the topological properties of a Lyapunov candidate function \eqref{eq:top_lyap}. Then a diffeomorphism is constructed such that the function under the diffeomorphic transformation $V_\phi(\cdot)$ satisfies the Lyapunov conditions on the samples \eqref{eq:dclop:L2}. 
This is convenient, since the geometric properties of a Lyapunov function are well known, and therefore, the surrogate function can be easily specified, e.g., to $V_b(\bm{x})=\bm{x}^\intercal\bm{x}$ for a single attractor system. 

\textbf{Note:} The distinct advantage of encoding geometric knowledge through a base function $V_b(\cdot)$ becomes even more apparent when considering general system classes with different attractor landscapes. Typical function approximation approaches, do not readily extend to more involved attractor landscapes, since the new \textit{Lyapunov-like} function requires different geometric constraints. In contrast, our proposed diffeomorphic learning framework merely requires an appropriate base function $V_b(\cdot)$, that encodes the topology of the desired attractor landscapes. 

\section{DEEP DIFFEOMORPHIC RBF NETWORK}
While the formulation in \eqref{eq:dclop} provides a convenient framework to infer Lyapunov functions through a diffeomorphic reformulation, it still requires optimizing over the space of diffeomorphisms $\pazocal{D}$. In particular, the evaluation of constraint \eqref{eq:dclop:L2} necessitates an analytical expression of the Jacobian $\bm{J}_\phi(\cdot)$ to evaluate the gradient $\nabla_{\bm{x}}^\intercal V_{\phi}(\cdot)$, which some popular invertible modelling frameworks, such as NODEs \citep{paper:chen18} or invertible ResNets \citep{behrmanninvertible} do not provide. While alternative approaches based on partitioned transformations with affine coupling layers (ACLs) \citep{dinh14, paper:dinh16, Kingma18, Rana20} admit an analytic expression of the Jacobian, this is at the cost of limited flexibility regarding coordinate-wise transformations.  

To this end, we propose a novel approach to construct diffeomorphisms using layers of bijective RBF maps, which we call \textbf{D}eep \textbf{D}iffeomorphic \textbf{RBF} \textbf{N}etwork 
(DD-RBFN). The principle architecture follows the residual learning paradigm \citep{resnet16}, where we use kernel machines to learn a residual component that is added to an identity mapping. By an appropriate choice of activation function, i.e., smooth function with analytical, partial derivative, such as a Gaussian kernel, we are able to derive simple box constraints for the network weights that guarantee invertibility of the learned map. \Cref{fig:rbf_layer} illustrates the working principle of our proposed DD-RBFN and \Cref{tab:comparison} compares our method to other invertible network architectures.

\newcommand{\cmark}{\ding{51}}%
\newcommand{\xmark}{\ding{55}}%
\renewcommand{\arraystretch}{1.2}
\begin{table}[t]
\caption{Comparison of diffeomorphism constructions.} \label{tab:comparison}
\begin{center}
\begin{footnotesize}
\begin{tabular}{l|cccc} 
\toprule
  & ACL & NODE & i-ResNet & \textbf{Our} \\
\midrule 
Analytic Forward & \cmark & \xmark & \cmark & \colorbox{green!10!white}{\cmark}\\
Analytic Jacobian & \cmark & \xmark & \xmark & \colorbox{green!10!white}{\cmark}\\
Freeform Jacobian & \xmark & \cmark & \cmark & \colorbox{green!10!white}{\cmark}\\
Sup-Universality & \cmark & \cmark & \textbf{N/A} & \colorbox{green!10!white}{\cmark}\\
\bottomrule
\end{tabular}
\end{footnotesize}
\end{center}
\end{table}

\begin{figure*}
    \centering
    \includegraphics[width=0.9\textwidth]{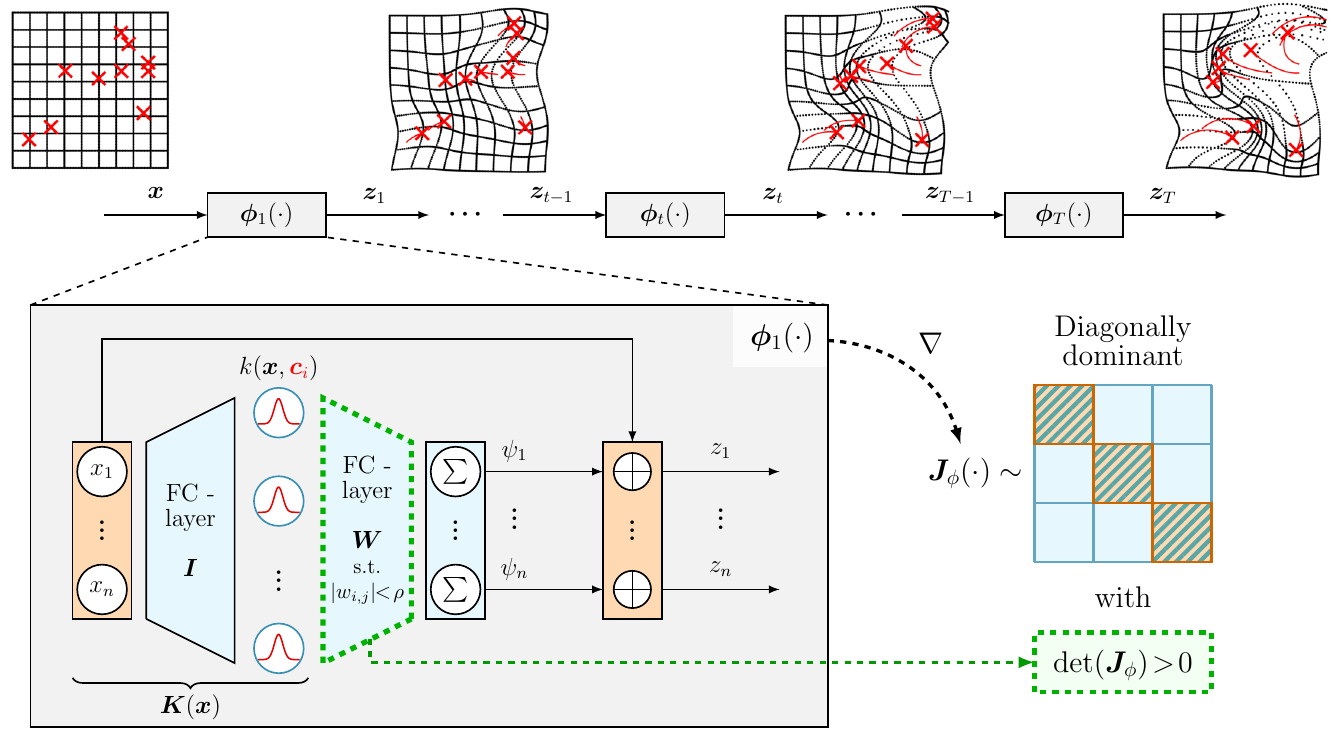}
    \caption{Depiction of the proposed \textbf{D}eep \textbf{D}iffeomorphic \textbf{RBF} \textbf{N}etwork (DD-RBFN) architecture. In each layer, a residual RBF mapping is applied which induces a state-space transformation $\bm{\phi}(\cdot)$ with a diagonally dominant Jacobian $\bm{J}_\phi(\cdot)$. Through appropriate weight bounds $\rho$, the bijectivity of the mapping is guaranteed. }
    \label{fig:rbf_layer}
\end{figure*}

\subsection{Bijective Residual RBF Layer}
Each layer of the DD-RBFN induces a mapping
\begin{align}
    \label{eq:theo:kerneldiffeo:model}
    \bm{\phi}(\bm{x}) &= \bm{x} + \bm{W}\bm{K}(\bm{x}) \\
    &= \bm{x}+\begin{bmatrix}\psi_1(\bm{x})\\ \vdots \\\psi_n(\bm{x})\end{bmatrix},
\end{align}	
with
\begin{equation}
    \label{eq:residual_sums}
    \bm{\psi}_j(\bm{x}) = \sum_i^N w_{i,j}k(\bm{x},\bm{c_i})
\end{equation}
where $N$ is the number of neurons in the hidden layer, $\bm{c}_i\in\mathbb{R}^n$ denotes the center point of the $i$-th neuron, $\bm{W}\in\mathbb{W}^{ n\times N}\!\!\subset\!\mathbb{R}^{n\times N}$ represents the weight matrix of the output layer with entries $w_{i,j}$, and $\bm{K}\in\mathbb{R}^{N\times 1}$ is the vector of radial basis functions 
with entries defined as $K_i(\bm{x}) = k(\bm{x}, \bm{c}_i)$. 
In particular, we deploy multi-variate Gaussian kernels as activation functions 
\begin{equation}
    \label{def:gaussiankernel}
    k(\bm{x}, \bm{c}_i) = \exp\Big(-\frac{1}{2}(\bm{x} - \bm{c}_i)^\top \bm{\Sigma}^{-1} (\bm{x} - \bm{c}_i)\Big),
\end{equation}
where $\bm{\Sigma}\in\mathbb{R}^{n\times n}$ is the positive-definite, symmetric covariance matrix. 
In order to guarantee bijectivity for the mapping \eqref{eq:theo:kerneldiffeo:model}, we make use of a finding regarding \textit{determinants of diagonally dominant matrices}.
\begin{theorem}[Determinant of Diagonally dominant matrices \citep{paper:ostrowski1937}]
	\label{theo:main:detbound}
	If $\bm{A} = \bm{I} - \bm{E}$ is a real $n\times n$ matrix with the elements of $\bm{E}$ being bounded in absolute value by~${\epsilon \leq \frac{1}{n}}$, then a lower bound of the determinant of $\bm{A}$ is given by:
	\begin{equation}
		\det(A) \geq 1 - n\epsilon
	\end{equation}
\end{theorem}
\textcolor{black}{By exploiting \Cref{theo:main:detbound}, we derive conditions for the weight matrix $\bm{W}$ under which the invertibility of our mapping \eqref{eq:theo:kerneldiffeo:model} is guaranteed.}
\begin{theorem}[Bijective Residual RBF Map]
    \label{theo:main:kerneldiffeo}
    Consider a mapping $\bm{\phi}(\cdot)$ as in \eqref{eq:theo:kerneldiffeo:model} 
    with multi-variate Gaussian kernel $k(\cdot,\cdot)$ \eqref{def:gaussiankernel} as an activation function. If the entries $w_{i,j}$ in the weight matrix $\bm{W}$ adhere to the box-constraints
	\begin{equation}
    \begin{split}
        \label{eq:main:kerneldiffeo:boxconstraint}
		| w_{i, j} | < &\frac{1}{nN\sum_{l=1}^{n}|Q_{j, l}|e^{-\frac{1}{2}}\sqrt{D_{j, j}}} := \rho(n, N, \bm{\Sigma}) \\ &\qquad\qquad\qquad \forall i \in [1, \dots, N], j \in [1, \dots, n],
    \end{split}
	\end{equation}
    where $\bm{Q}$ and $\bm{D}$ are due to the eigendecomposition $\bm{\Sigma}^{-1} = \bm{Q}\bm{D}\bm{Q}^\top$, then $\bm{\phi}(\cdot)$ is a diffeomorphism.
\end{theorem}

\begin{proof}
To enforce bijectivity of the mapping $\bm{\phi}(\cdot)$, we make use of the residual structure in \eqref{eq:theo:kerneldiffeo:model}, since it induces a Jacobian $\bm{J}_{\phi}(\cdot)$ that is decomposable into an identity matrix $\bm{I}_n$ and a \textit{disturbance} matrix $\bm{E}\in\mathbb{R}^{n\times n}$
\begin{equation}
    \begin{split}
        \bm{J}_\phi(\bm{x}) &= \begin{bmatrix} 
        1 + \frac{\partial \bm{\psi}_1(\bm{x})}{\partial x_1} & \frac{\partial \bm{\psi}_1(\bm{x})}{\partial x_2} & \dots & \frac{\partial \bm{\psi}_1(\bm{x})}{\partial x_n}\\
        \frac{\partial \bm{\psi}_2(\bm{x})}{\partial x_1} & 1 + \frac{\partial \bm{\psi}_2(\bm{x})}{\partial x_2} & & \vdots \\
        \vdots & & \ddots & \vdots \\
        \frac{\partial \bm{\psi}_n(\bm{x})}{\partial x_1} &        & & 1 + \frac{\partial \bm{\psi}_n(\bm{x})}{\partial x_n} 
    \end{bmatrix} \\ &:= \bm{I}_n - \bm{E}(\bm{x}).
    \end{split}
\end{equation}
\textcolor{black}{From Theorem~\ref{theo:main:detbound} we have that the determinant $\det(\bm{J}_\phi)$ of such a decomposable matrix can be bounded from below 
\begin{equation}
    \det(\bm{J}_\phi(\bm{x})) \geq 1 - n\epsilon > 0,
\end{equation}}
if the absolute value of the elements of the disturbance matrix $\bm{E}$ admit a bound $\epsilon$ as such: 
\begin{equation}
\label{eq:epsilon_bound}
    \epsilon = \max_{l,j}(\| E_{l,j}\|) \leq \frac{1}{n}, \quad \forall l \in [1, \dots, n], j \in [1, \dots, n].
\end{equation}
Since the elements of $\bm{E}$ consist of partial derivatives of the residual terms \eqref{eq:residual_sums}, i.e., ${E_{l,j}=-\frac{\partial \psi_l(\bm{x})}{\partial x_j}}$, the inequality \eqref{eq:epsilon_bound} directly translates to a bound on the weighted sum of partial derivatives of the RBFs 
\begin{equation}
    \label{eq:sum_partial_derivative_bound}
    \max_{i,j}\Big(\sum_{i=1}^{N} \left|w_{i, j}\right| \, \left\|\frac{\partial k(\bm{x},\bm{c_i})}{\partial x_j}\right\|\Big) \leq \frac{1}{n}.
\end{equation}
For the choice of Gaussian kernel $k(\cdot,\cdot)$ \eqref{def:gaussiankernel} 
the maximum value of $\|\frac{\partial k}{\partial x_j}\|$ is bounded, can be computed in closed-form, and only depends on the entries of $\bm{\Sigma}$. Thus, it suffices to appropriately bound the weights $|w_{i, j}|$ 
to guarantee the positive-definiteness of $\det(\bm{J}_{\phi})$ and thereby the invertibility of the mapping $\bm{\phi}(\cdot)$. 
\end{proof}
\Cref{theo:main:kerneldiffeo} provides a bound on the maximum, relative change that the state-space transformation $\bm{\phi}(\cdot)$ may induce 
such that 
bijectivity is ensured. 
The corresponding weight bound \eqref{eq:main:kerneldiffeo:boxconstraint} is inversely proportional to the state dimension $n$, the number of RBF centers $N$, and the shape of the RBF activation functions encoded by $\bm{\Sigma}$. Therefore, intuitively the admissible deformation decreases, if there are many basis functions with small bandwidths, since many local changes are more likely to induce a non-bijective mapping. Additionally, the eigendecomposition of $\bm{\Sigma}^{-1}$ in \Cref{theo:main:kerneldiffeo} facilitates bounding the maximum value of $\|\frac{\partial k}{\partial x_j}\|$ for the more general symmetric covariance case, as illustrated in Figure \ref{fig:decomp}.
\begin{figure}
    \centering
    \includegraphics[width=0.45\textwidth]{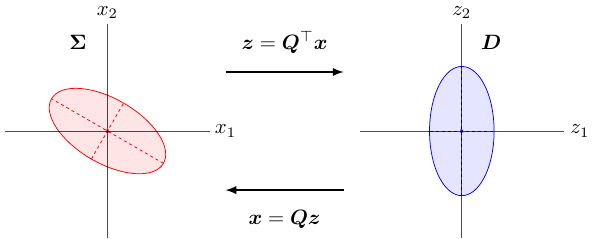}
    \caption{Visualization of bound derivation for Gaussian kernel with symmetric covariance. The kernel with diagonal covariance $\bm{D}$  is equivalent to $k$ under a state-space transformation. This transformation is given by the matrix $\bm{Q}$ under which the original~$\bm{\Sigma}$ is axis-aligned with the new coordinate axis $\bm{z} = \bm{Q}^\top\bm{x}$.}
    \label{fig:decomp}
\end{figure}

\textbf{Note:} The derived constraint \eqref{eq:main:kerneldiffeo:boxconstraint} 
is conservative, as \Cref{theo:main:kerneldiffeo} imposes 
the same absolute bound on all coefficients instead of considering the weighted sum as shown in \eqref{eq:sum_partial_derivative_bound}. Advantageously, this leads to 
simple box-constraints that only depend on a-priori known quantities facilitating a linear optimization problem. Critically, the bound for each weight $w_{i,j}$ is independent of the other coefficients, thereby admitting an optimization over all weights in a distributed fashion.



\subsection{Composition of Network Layers}
While the formulation in \Cref{theo:main:kerneldiffeo} is convenient from an optimization perspective, it also imposes a limitation on the admissible network weights and thereby restricts the expressivity of each bijective RBFN layer. To overcome this, we exploit a property of invertible maps, i.e., that bijectivity is preserved under composition. Since the composition of smooth functions is trivially smooth, the same holds for diffeomorphisms, which is formalized in the following.

\begin{proposition}{	\label{prop:diffeocomp}
The composition of $T$ orientation preserving diffeomorphisms 
	\begin{equation}
		\label{eq:prop:kernelcomp}
		\bm{\Phi}(\bm{x}) \coloneqq \bigcomp_{t=1}^T \bm{\phi}_t(\bm{x}), \quad t \in \mathbb{N}_+, 
	\end{equation}
is also an orientation preserving diffeomorphism.}
\end{proposition}

Thus, we propose to construct compositions of bijective RBFN layers to obtain a more expressive mapping resulting in the DD-RBFN architecture:  
\begin{equation}
\begin{split}
    	\label{eq:rbfn_comp}
	\bm{\Phi}_{\text{DD-RBFN}}(\bm{x}) &\coloneqq \bigcomp_{t=1}^T \underbrace{\bm{\phi}_{t-1}(\bm{x}) + \bm{W}_{t}\bm{K}_{t}(\bm{\phi}_{t-1}(\bm{x}))}_{\coloneqq \bm{\phi}_{t}(\bm{x})}, \\ &\enspace \text{s.t. } |w_{i, j, t}| \leq \rho_t(n, N_t, \bm{\Sigma}_t),
\end{split}
\end{equation}
with $\bm{\phi}_{0}(\cdot)=\bm{I}$ and where $t$ denotes the iterator variable over the layers. For each $\bm{\phi}_{t}(\cdot)$, the number of center points $N_t$, the center positions $\bm{c}_{t,i}$, and covariance matrix $\bm{\Sigma}_t$ can be chosen freely. Since the coefficient bounds $\rho_t(\cdot)$ in \eqref{eq:main:kerneldiffeo:boxconstraint} are independent of the center positions, recomputing them is only necessary if $N_t$ or $\bm{\Sigma}_t$ changes. \Cref{fig:rbf_layer} (top) illustratively depicts the DD-RBFN architecture, where we denote the output of the composition at the $t$-th layer by $\bm{z}_t\in\mathbb{R}^n$, e.g., $\bm{z}_0=\bm{x}$ and $\bm{z}_2=\bm{\phi}_2(\bm{\phi}_1(\bm{x}))$, for improved readability.


The composite map \eqref{eq:rbfn_comp} can also be interpreted as the evolution of a time-varying, discrete-time system
\begin{equation}
    \label{eq:discrete_sys}
    \bm{z}_{t+1} = \bm{z}_{t} + \bm{W}_{t}\bm{K}_{t}(\bm{z}_{t}), \quad \text{s.t. } |w_{i, j, t}| \leq \rho_t(n, N_t, \bm{\Sigma}_t),
\end{equation}
where the dynamics of the difference equation \eqref{eq:discrete_sys} are driven by the weight matrix $\bm{W}_t$ and the gram matrix $\bm{K}_t(\cdot)$. While some approaches exploit the existence and uniqueness of solutions to Lipschitz continuous ordinary differential equations \citep{book:khalil:nonlin, paper:chen18} to construct diffeomorphims, this property does not hold for discrete-time difference equations in general. Hence, intuitively the weight constraints $\rho_t$ induce a Lipschitz bound such that the difference equation \eqref{eq:discrete_sys} has a unique solution everywhere.

\subsection{Flow Endpoints and Sup-Universality}
In this section, we analyze the approximation capabilities of DD-RBFNs and demonstrate that our proposed architecture is a $\sup$-universal approximator 
for 
$\pazocal{C}^2$ diffeomorphisms.

A well-known result in the literature is that sufficiently wide, one-layer RBFNs are universal approximators on a compact subset of $\mathbb{R}^n$, under mild conditions 
\citep{Park91}. Similarly, linear combinations of universal kernels, such as Gaussian kernels \eqref{def:gaussiankernel}, are known to approximate continuous functions arbitrarily well \citep{Micchelli06}. This property extends to the residual-based mapping $\bm{\phi}(\cdot)$ in \eqref{eq:theo:kerneldiffeo:model}.
\begin{lemma}
	\label{lem:unconstrainedldk}
	Let $\phi: \mathbb{R}^n \to \mathbb{R}^n$ be defined as
	\begin{equation}
		\label{theo:uapunconstrained}
		\phi(\bm{x}) = \bm{x} + \bm{W}\bm{K}(\bm{x})
	\end{equation}
	with $\bm{W} \in \mathbb{R}^{n \times N}$ and denote by $\pazocal{M}_{\text{}}$ the set of all possible $\phi$ as defined in~\eqref{theo:uapunconstrained}. Then, $\pazocal{M}_{\text{}}$ is a $\sup$-universal approximator for Lipschitz continuous maps~$\pazocal{L}(\mathbb{R}^n)$. 
\end{lemma}
While Lemma \ref{lem:unconstrainedldk} establishes the principle expressiveness of a kernel-based residual mapping, it does not readily extend to the bijective formulation in \Cref{theo:main:kerneldiffeo} due to the weight constraints \eqref{eq:main:kerneldiffeo:boxconstraint}. More specifically, the mapping is Lipschitz bounded 
with 
\begin{equation}
    \label{eq:lip_bound}
    L \propto \sum_{i=1}^{N}|w_i| \left\|\frac{\partial k(\bm{x}, \bm{x}_i)}{\partial x_j}\right\|.
\end{equation}
Since the partial derivative of a Gaussian kernel is bounded, the box constraints on $w_i$ directly translate into a maximum value of $L$.
In \citep{paper:teshima20nodes}, it is established that any~$\pazocal{C}^2$-diffeomorphism can be approximated over a compact set by a composition of flows generated by Lipschitz continuous dynamical system $f \in \pazocal{L}(\mathbb{R}^n)$. This property extends to our DD-RBFN architecture, despite the Lipschitz bound \eqref{eq:lip_bound}.

\begin{theorem}[$\sup$-universality of DD-RBFNs]
    \label{theo:sub_univ}
	Let $\phi_t \in \pazocal{M}_{\rho}$ define the set of all functions that satisfy \Cref{theo:main:kerneldiffeo}. Then, the composite model class
	\begin{equation}
        \label{eq:class_ddrbfn}
		\pazocal{M}_{\text{DD-RBFN}} = \Big\{ \bigcomp_{t=1}^T \phi_t | \phi_t \in \pazocal{M}_{\rho}, k \in \mathbb{N}_+ \Big\}
	\end{equation}
	is a $\sup$-universal approximator for $\pazocal{C}^2$ diffeomorphisms.
\end{theorem}
\begin{proof}
Exploiting the analogous interpretation of the DD-RBFN architecture as a discrete-time system, the composite mapping \eqref{eq:discrete_sys} is equilivant to an Euler discretization of a continuous-time ODE 
\begin{equation}
    \label{eq:euler_discrete_sys}
    \bm{z}_{t+1} = \bm{z}_{t} + \delta\bm{W}_{t}\bm{K}_{t}(\bm{z}_{t}), 
\end{equation}
with newly introduced virtual time-step $\delta\in(0,1]$. Due to $\delta$, the admissible weight bound is adjusted to
\begin{equation}
    \rho_\delta(n,N,\bm{\Sigma},\delta) = \frac{1}{|\delta|}\rho(n,N,\bm{\Sigma}),
\end{equation}
where $\rho_\delta$ is the virtual time-step dependent weight bound. Since the coefficient bound $\rho_\delta$ becomes arbitrary large in the limit $\delta\to0$, the instantaneous integration step of any $f \in \pazocal{L}(\mathbb{R}^n)$ can be approximated given a sufficiently small virtual time-step $\delta$. Thus, Lemma \ref{lem:unconstrainedldk} applies for each layer $\phi_t \in \pazocal{M}_{\rho}$ in \eqref{eq:class_ddrbfn}. Consequently, a sufficiently long composition $\pazocal{M}_{\text{DD-RBFN}}$ can approximate arbitrary flows of Lipschitz continuous ODEs. 
\end{proof}
The remainder of the proof is based on the fact that locally supported diffeomorphisms can
be described by a finite composition of flow endpoints generated by ODEs integrated over unit time \citep{paper:teshima20nodes, Teshima23}. The implication of considering a virtual time-step $\delta$ in \eqref{eq:euler_discrete_sys} is illustratively depicted in \Cref{fig:delta_limit}. Note that the required network depth to reach a flow endpoint increases with decreasing $\delta$, since each integration step is realized by a single layer.

\begin{figure}
    \centering
    \includegraphics[width=0.48\textwidth]{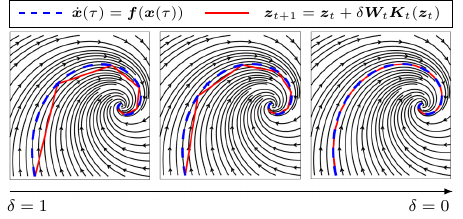}
    \caption{Illustration of virtual time-step $\delta$ and resulting implication on the admitted flow-endpoints.}
    \label{fig:delta_limit}
\end{figure}

\subsection{Learning Diffeomorphic Lyapunov Functions from Data} %
To finally obtain a Lyapunov function from data, we solve the diffeomorphic optimization problem \eqref{eq:dclop} using our proposed architecture. To this end, we exploit the analogous interpretation of DD-RBFN as a time-discrete system to solve the diffeomorphic optimization problem using tools from optimal control theory. 

\begin{algorithm}[t]
\small
\caption{MPC-based DD-RBFN optimization}\label{alg:cap}
\begin{algorithmic}[1]
\Require $V_b(\bm{x})$, $\mathbb{D}$, $\bm{\Sigma}$, $H$
\For{$t=0, \dots, \texttt{\#Iterations}$}
    \State $\{\bm{W}_{t+1},\ldots,\bm{W}_{t+H}\} \gets $ \texttt{solve} \eqref{eq:mpc_diffeo} 
    \State $\bm{\Phi} \gets \bm{\phi}_{t+1}\circ\bm{\phi}_t$ \Comment{\texttt{Update} $\bm{\Phi}$ \texttt{with} $\bm{W}_{t+1}$ only}
    \State $V_{\bm{\Phi}} \gets V_b \circ \bm{\Phi}$ \Comment{\texttt{Update Lyapunov function}}
    \State \texttt{discard} $\{\bm{W}_{t+2},\ldots,\bm{W}_{t+H}\}$
\EndFor
\State \Return $\bm{\Phi}$, $V_{\bm{\Phi}}$
\end{algorithmic}
\label{ch3:alg:sequence}
\end{algorithm}

\textcolor{black}{Specifically, each network layer $t$ corresponds to a time-step, the output at each layer $\bm{z}_t$ represents the system's state and the network weights $\bm{W}_t$ correspond to control inputs. Together with the weight constraints $\rho_t$, it results in an input-constrained system.} We deploy \textit{model predictive control} (MPC), which is a common method to solve nonlinear, non-convex problems by iteratively optimizing the control inputs over a prediction horizon $H$ \citep{book:grune2017nonlinear, Kamthe2018}. \textcolor{black}{Since MPC facilitates input constraint handling, it is convenient to deploy here.} Hence, at each layer, we optimize over the weights 
	\begin{subequations}
        \label{eq:mpc_diffeo}
		\begin{align}
			\label{eq:mpc_diffeo_obj}
			\bm{W}_{t+1} &= \argmin_{\bm{W}_{t+1},\ldots,\bm{W}_{t+H}} \sum_{h=1}^{H}\sum_{i=1}^{N} \nabla_{\bm{x}}^\intercal V_{\Phi}(\bm{z}_{t+h}^{(i)}) \dot{\bm{x}}^{(i)} \\
			\label{eq:mpc_diffeo_c1}
			&\text{ s.t.} \quad\quad\, \bm{z}_{t+1} = \bm{z}_t + \bm{W}_{t}\bm{K}_{t}(\bm{z}_{t}) \\ 
            \label{eq:mpc_diffeo_c2}
            & \quad\quad |w_{i,j,t+h}| < \rho_{t+h} \quad \forall h \in [1, \dots, H]  \\
            \label{eq:mpc_diffeo_c3}
            & \qquad\qquad\, \bm{z}_{0}^{(i)} = \bm{x}^{(i)} 
		\end{align}
	\end{subequations}
where \eqref{eq:mpc_diffeo_c1} represents the dynamics and \eqref{eq:mpc_diffeo_c2} the input constraints. The resulting $\bm{W}_{t+1}$ is used to iteratively append the DD-RBFN and by minimizing the empirical Lyapunov risk \eqref{eq:mpc_diffeo_obj} \citep{Chang19}, the diffeomorphic optimization is guided in a manner that induces constraint satisfaction on the samples. 

\section{EVALUATION}
To evaluate our proposed approach, we first apply our method to different attractor landscapes, including point attractors, two-attractor systems, and limit cycles. Second, we demonstrate the expressiveness of our approach in a comparison evaluation with related works. An implementation of \Cref{ch3:alg:sequence} is available at \url{https://github.com/stesfazgi/Diffeomorphic-Lyapunov-Functions} \looseness=-1 


\begin{figure*}[t]
        \centering
		\includegraphics[trim={0cm 0.11cm 0cm 0cm}, clip, width=0.8\textwidth]{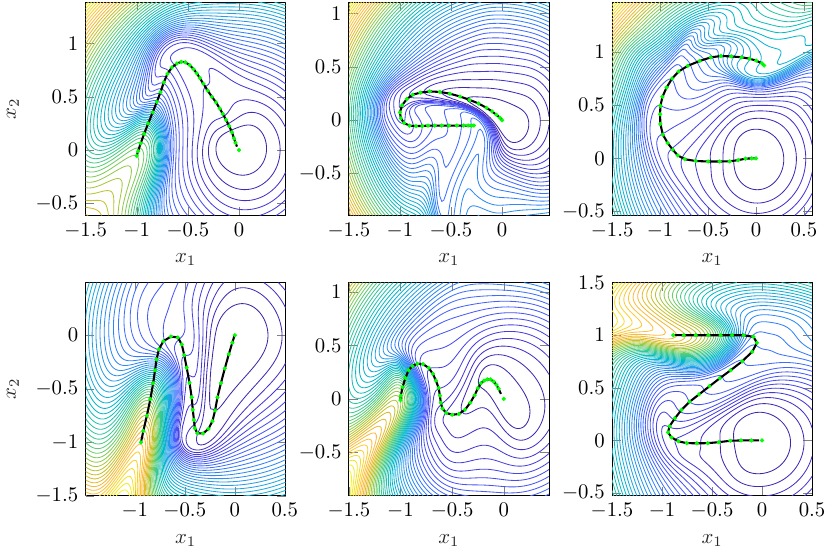} \vspace{-0.2cm}
	\caption[Comparison with WSAQF Lyapunov functions for different LASA shapes]{Successfully learned diffeomorphic Lyapunov functions for $6$ exemplary shapes of the LASA dataset with data trajectories (black) and  samples on which the Lyapunov conditions are satisfied (green).} 
	\label{fig:eval_comp_dmp}
\end{figure*}

\subsection{Attractor Landscapes} 
\textbf{Single Equilibrium.}
For this evaluation, we use the LASA handwriting dataset \citep{khansari2011}, which is a popular benchmark in the learning stable dynamical system literature \citep{perrin2016fast, Rana20, Zhang23}. 
\Cref{fig:eval_comp_dmp} shows the result of applying the diffeomorphic learning approach to 6 exemplary shapes with a base function $V_b(\bm{x}) = 0.1\bm{x}^\top\bm{x}$. Our proposed approach successfully learns a diffeomorphism such that the transformed function $V_\Phi(\cdot)$ constitutes a valid Lyapunov function for the trajectories.  

\textbf{Two Attractor System.} 
Additionally, we deploy the approach on a dynamical system with two stable equilibria and one unstable equilibrium as depicted by the vector field in \Cref{fig:eval_meq_problemsetup} (left). For the training data, the system is initialized at 6 different positions and simulated to convergence. 
The obtained $V_\Phi(\cdot)$ 
is depicted in \Cref{fig:eval_meq_problemsetup} (right). 
The learned diffeomorphic Lyapunov function is consistent
with all demonstrations included in the dataset and successfully identified the position of the two stable equilibria. \looseness=-1

\begin{figure}[H]
    \centering
    \begin{tikzpicture}
        \node[inner sep=0pt] (russell) at (0,0)
        {\includegraphics[trim={1cm 0.5cm 0cm 0cm}, clip,width=.21\textwidth]{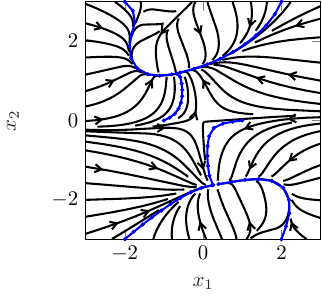}};
        \node[inner sep=0pt] (whitehead) at (4,0)
        {\includegraphics[trim={1cm 0.5cm 0cm 0cm}, clip,width=.21\textwidth]{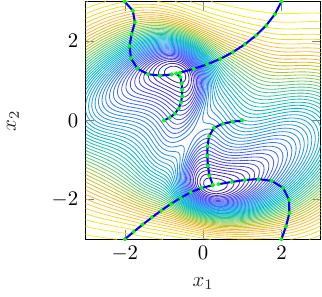}};
    \end{tikzpicture}\vspace{-0.2cm}
    \caption{Two equilibria system. \textbf{Left:} Vector field of the system (black) and data samples (blue). 
    \textbf{Right:} Learned diffeomorphic Lyapunov-like function.} 
    \label{fig:eval_meq_problemsetup}
\end{figure}

\begin{table*}[]
    \centering
    \caption{Resulting mean and standard deviation of the violation rate, i.e., the percentage of data points on which the Lyapunov conditions are violated (smaller=better), over 6 shapes of the LASA handwriting dataset.}
    \footnotesize
    \begin{tabular}{l | c c c c c c c}
        \toprule
                 & MLP & PD-MLP & ACL & i-ResNet & NODE & \textbf{DD-RBFN (Our)} \\
        \midrule
         Angle & $2.14\pm0.86$ & $2.25\pm1.86$ & $2.80\pm1.61$ & $9.74\pm9.19$ & $2.58\pm1.82$ & $\bm{2.00\pm0.71}$ \\
         C-Shape & $7.29\pm5.35$ & $3.20\pm0.26$ & $5.58\pm2.13$ & $4.06\pm0.41$ & $2.48\pm0.45$ & $\bm{2.45\pm0.86}$ \\
         Z-Shape & $9.95\pm7.01$ & $8.44\pm2.14$ & $6.65\pm3.92$ & $3.43\pm4.03$ & $7.02\pm3.89$ & $\bm{3.26\pm2.41}$ \\
         N-Shape & $15.05\pm6.52$ & $11.69\pm3.12$ & $12.45\pm2.79$ & $23.79\pm1.92$ & $14.69\pm4.76$ & $\bm{10.89\pm1.30}$ \\
         Sine & $11.52\pm8.54$ & $4.55\pm3.10$ & $2.84\pm1.18$ & $13.95\pm11.62$ & $4.18\pm3.52$ & $\bm{1.41\pm0.50}$ \\
         Bended line & $33.03\pm11.71$ & $15.66\pm2.07$ & $\bm{7.78\pm7.51}$ & $38.85\pm2.70$ & $20.02\pm15.14$ & $8.44\pm4.38$ \\
         \bottomrule
    \end{tabular}
    \label{tab:comparison}
\end{table*}

\textbf{Limit Cycle System.} 
Finally, we use 
the proposed approach to find a Lyapunov-like function for a 
limit cycle system. To this end, we consider the well-known Van der Pol oscillator \citep{pol26} depicted in \Cref{fig:eval_limit} (left) with the stable limit cycle highlighted in red. For training, we sample 20 approximately equally spaced data points along the limit cycle and simulate 4 trajectories starting in the corners of the state space. The sampled data and the resulting diffeomorphic function $V_\Phi(\cdot)$ are shown in \Cref{fig:eval_limit} (right). It can be seen that the zero gradient contour line of $V_\Phi(\cdot)$ (marked in red) aligns well with the data sampled along the limit cycle (blue dots). Additionally, the Lyapunov constraints are fulfilled along the trajectories converging towards the limit cycle.

\begin{figure}[H]
    \centering
    \begin{tikzpicture}
        \node[inner sep=0pt] (russell) at (0,0)
        {\includegraphics[trim={1cm 0.5cm 0cm 0cm}, clip,width=.21\textwidth]{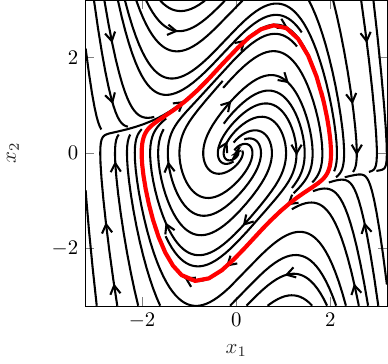}};
        \node[inner sep=0pt] (whitehead) at (4,0)
        {\includegraphics[trim={1cm 0.5cm 0cm 0cm}, clip,width=.215\textwidth]{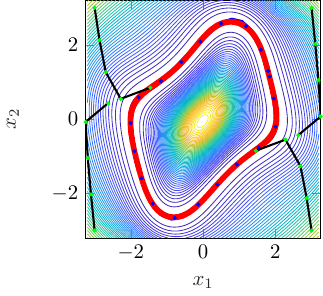}};
    \end{tikzpicture}\vspace{-0.2cm}
    \caption{\textbf{Left:} Vector field of system (black) with limit cycle (red). \textbf{Right:} Contour lines of learned diffeomorphic Lyapunov-like function with sampled data.} 
    \label{fig:eval_limit}
\end{figure}

\subsection{Comparison Evaluation} 

For the comparison evaluation, we again consider the single-equilibrium case using the LASA dataset. In the dataset, each shape consists of a total of 7 demonstrations with 1000 data points each. We train the methods on a subset of the data of one demonstration and evaluate the generalization capabilities of the learned Lyapunov function on the remaining 6 demonstrations by computing a violation rate, i.e., the percentage of data points on which the Lyapunov conditions are violated. We repeat this process for 5 different seeds. For the comparison, we use other diffeomorphism constructions, i.e., ACL \citep{paper:dinh16}, i-ResNets \citep{behrmanninvertible}, and NODEs \citep{paper:chen18}. Moreover, we include two MLP-based Lyapunov function approximators, similar to the learner deployed in \citep{Chang19} and \citep{dawson2021}. In \citep{dawson2021}, the authors guarantee the positive definiteness of the output by taking the inner product of the last hidden layer. Therefore, we call the architecture PD-MLP in the following. 

\Cref{tab:comparison} reports the obtained results for the 6 shapes shown in \Cref{fig:eval_comp_dmp}. 
With the exception of the Bended line, the PD-MLP architecture outperforms the naive MLP, which indicates the benefit of prior structural knowledge. Moreover, the generative approaches show a similar performance as the PD-MLP, which indicates the principal expressivity of the indirect, diffeomorphic Lyapunov learning framework. 
Within the diffeomorphic methods, our DD-RBFN-based approach achieves the best performance for all of the shapes except the Bended line, for which ACL has a slightly lower violation rate.

\section{Conclusion}
In this work, we present a diffeomorphic learning framework to infer Lyapunov functions from data. 
By designing simple base functions and optimizing over topology-preserving maps, we successfully encode prior geometrical knowledge during inference. 
Furthermore, we propose a novel, diffeomorphic modeling approach based on RBFs to facilitate precise, data-driven transformations. Finally. we evaluate our method on systems with different attractor landscapes, including multiple equilibria and limit cycles, and in a comparison evaluation with related works.  

\subsubsection*{Acknowledgements}
This work was supported by the European Research Council (ERC) Consolidator grant “CO-MAN” under grant agreement no. 864686.

\bibliography{mybib}

\section*{Checklist}

 \begin{enumerate}

 \item For all models and algorithms presented, check if you include:
 \begin{enumerate}
   \item A clear description of the mathematical setting, assumptions, algorithm, and/or model. [Yes]
   \item An analysis of the properties and complexity (time, space, sample size) of any algorithm. [Yes. Analyzed approximation capabilities.]
   \item (Optional) Anonymized source code, with specification of all dependencies, including external libraries. [Yes.]
 \end{enumerate}

 \item For any theoretical claim, check if you include:
 \begin{enumerate}
   \item Statements of the full set of assumptions of all theoretical results. [Yes]
   \item Complete proofs of all theoretical results. [Yes]
   \item Clear explanations of any assumptions. [Yes]     
 \end{enumerate}

 \item For all figures and tables that present empirical results, check if you include:
 \begin{enumerate}
   \item The code, data, and instructions needed to reproduce the main experimental results (either in the supplemental material or as a URL). [Yes.]
   \item All the training details (e.g., data splits, hyperparameters, how they were chosen). [Yes]
         \item A clear definition of the specific measure or statistics and error bars (e.g., with respect to the random seed after running experiments multiple times). [Not Applicable]
         \item A description of the computing infrastructure used. (e.g., type of GPUs, internal cluster, or cloud provider). [No]
 \end{enumerate}

 \item If you are using existing assets (e.g., code, data, models) or curating/releasing new assets, check if you include:
 \begin{enumerate}
   \item Citations of the creator If your work uses existing assets. [Yes]
   \item The license information of the assets, if applicable. [Not Applicable]
   \item New assets either in the supplemental material or as a URL, if applicable. [Not Applicable]
   \item Information about consent from data providers/curators. [Not Applicable]
   \item Discussion of sensible content if applicable, e.g., personally identifiable information or offensive content. [Not Applicable]
 \end{enumerate}

 \item If you used crowdsourcing or conducted research with human subjects, check if you include:
 \begin{enumerate}
   \item The full text of instructions given to participants and screenshots. [Not Applicable]
   \item Descriptions of potential participant risks, with links to Institutional Review Board (IRB) approvals if applicable. [Not Applicable]
   \item The estimated hourly wage paid to participants and the total amount spent on participant compensation. [Not Applicable]
 \end{enumerate}

 \end{enumerate}

\end{document}